\documentclass[authoryear,review,onefignum,onetabnum]{siamart190516}

\usepackage{enumerate}
\usepackage{bbm}
\usepackage{float}
\usepackage[shortlabels]{enumitem}

\nolinenumbers

\newcommand\Item[1][i]{%
	\ifx\relax#1\relax  \item \else \item[#1] \fi
	\abovedisplayskip=0pt\abovedisplayshortskip=0pt~\vspace*{-\baselineskip}}

\usepackage{lipsum}
\usepackage{amsfonts}
\usepackage{graphicx}
\usepackage{epstopdf}
\usepackage{algorithmic}
\ifpdf
  \DeclareGraphicsExtensions{.eps,.pdf,.png,.jpg}
\else
  \DeclareGraphicsExtensions{.eps}
\fi

\newsiamremark{remark}{Remark}
\newsiamremark{hypothesis}{Hypothesis}
\crefname{hypothesis}{Hypothesis}{Hypotheses}
\newsiamthm{claim}{Claim}

\headers{The non-tightness of the reconstruction threshold}{W. Liu and N. Ning}

\title{The non-tightness of the reconstruction threshold of a $4$ states symmetric model with different in-block and out-block mutations}

\author{Wenjian Liu\thanks{Dept.of Mathematics and Computer Science,
		Queensborough Community College, City University of New York 
  (\email{wjliu@qcc.cuny.edu}).}
\and Ning Ning\thanks{Dept. of Applied Mathematics, University
	of Washington, Seattle
  (\email{ningnin@uw.edu}).}}

\usepackage{amsopn}

\newtheorem*{Main Theorem}[theorem]{Main Theorem}{\normalfont\bfseries}{\itshape}
\ifpdf
\hypersetup{
  pdftitle={An Example Article},
  pdfauthor={D. Doe, P. T. Frank, and J. E. Smith}
}
\fi

\begin{document}

\maketitle

\begin{abstract}
The tree reconstruction problem is to collect and analyze massive data at the $n$th level of the tree, to identify whether there is non-vanishing information of the root, as $n$ goes to infinity. Its connection to the clustering problem in the setting of the stochastic block model, which has wide applications in machine learning and data mining, has been well established. For the stochastic block model, an ``information-theoretically-solvable-but-computationally-hard" region, or say ``hybrid-hard phase", appears whenever the reconstruction bound is not tight of the corresponding reconstruction on the tree problem. Although it has been studied in numerous contexts, the existing literature with rigorous reconstruction thresholds established are very limited, and it becomes extremely challenging when the model under investigation has $4$ states (the stochastic block model with $4$ communities). In this paper, inspired by the newly proposed $q_1+q_2$ stochastic block model, we study a $4$ states symmetric model with different in-block and out-block transition probabilities, and rigorously give the conditions for the non-tightness of the reconstruction threshold.
\end{abstract}

\begin{keywords}
Reconstruction, Markov random fields on trees, Deep generative hierarchical model, Unsupervised learning, Phase transition
\end{keywords}

\begin{AMS}
  60K35 \and 62F15 \and 82B20 \and 68R01
\end{AMS}

\section{Introduction}
\label{sec1a}\vspace{-4mm}
\subsection{The tree reconstruction problem}
The tree reconstruction problem, as an interdisciplinary subject, has been studied in numerous contexts including statistical physics, information theory, and computational biology. The reconstructability plays a crucial role in phylogenetic reconstruction in evolutionary
biology (see, for instance, \cite{mossel2004phase, daskalakis2006optimal}), communication theory in the study of noisy computation (see, for instance, \cite{evans2000broadcasting}), analogous investigations in the realm of network tomography (see, for instance, \cite{bhamidi2010network}), reconstructability and distinguishability in the clustering problem of the stochastic block model (see, for instance, \cite{mossel2013proof, mossel2014belief, neeman2014non, banks2016information, brito2016recovery}), etc. 

The tree reconstruction model has two building blocks, with one being an irreducible aperiodic Markov chain on a finite characters set $\mathcal{C}$ and the other one being a rooted $d$-ary tree (every vertex having exactly $d$ offspring).
The tree is denoted as $\mathbb{T}=(\mathbb{V}, \mathbb{E}, \rho)$, where $\mathbb{V}$ stands for vertices, $\mathbb{E}$ stands for edges, and $\rho\in \mathbb{V}$ stands for the root. Denote $\sigma_v$ as the state assigned to vertex $v$, and denote $\sigma_\rho$ specially for the state of the root $\rho$ that is chosen according to an initial distribution $\pi$ on $\mathcal{C}$.  The root signal propagates in the tree according to a transition matrix $\mathbf{M}$ which is also called noisy channel, in a way that for each vertex $v$ having $u$ as its parent, the spin/configuration at $v$ is assigned according to the probability $M_{i j}=\mathbf{P}(\sigma_v=j\mid\sigma_u=i)$ for $i, j \in \mathcal{C}$.

The reconstruction problem on an infinite tree is to analyze that given the configurations realized at the $n$th layer of the tree which is denoted as $\sigma(n)$, whether there exists non-vanishing information on the letter transmitted by the
root, as $n$ goes to infinity. Based on $\sigma^i(n)$ which is defined as $\sigma(n)$ conditioned on $\sigma_\rho = i$, the following definition gives one mathematical formulation on reconstructibility:
\begin{definition}
	We say that 
	a model is \textup{\textbf{reconstructible}} on an infinite tree $\mathbb{T}$, if for some $i, j\in \mathcal{C}$
	$$
	\limsup_{n\to \infty}d_{TV}(\sigma^i(n), \sigma^j(n))>0,
	$$
	where $d_{TV}$ is the total variation distance. When the $\limsup$
	is $0$, we say that the model is \textup{\textbf{non-reconstructible}} on
	$\mathbb{T}$.
\end{definition}

\subsection{Existing results with states other than $4$}
The reconstructibility is closely related to, the second
largest eigenvalue by absolute value of the transition matrix $\mathbf{M}$, denoted as $\lambda$. 
It is well known that the reconstruction problem is solvable when $d\lambda^2>1$ which is the Kesten-Stigum bound (\cite{kesten1966additional,kesten1967limit}), however when $d\lambda^2 < 1$ the problem becomes much more challenging and its solvability highly depends on the channel. 

The binary model with $2$ states corresponds to the Ising model in statistical physics, whose transition matrix is given by
$$
\mathbf{M}= \frac{1}{2} \left(
\begin{array}{cc}
1+\theta &\;  1-\theta \\
1-\theta &\;  1+\theta \\
\end{array}
\right) + \frac{\Delta}{2}\left(
\begin{array}{cc}
-1 &\; 1 \\
-1 &\; 1 \\
\end{array}
\right), \quad\quad |\theta|+|\Delta|\leq 1,
$$ 
where $\Delta$ is used to describe the deviation from the symmetric channel, i.e. when $\Delta \neq 0$ the channel is asymmetric.
For the binary symmetric channel, \cite{bleher1995purity} showed that the
reconstruction problem is solvable if and only if $d\lambda^2>1$. For the binary asymmetric channel with sufficiently large asymmetry, \cite{mossel2001reconstruction, mossel2004survey}  showed that the Kesten-Stigum bound is not the bound for reconstruction. When the asymmetry is sufficiently small, \cite{borgs2006kesten} established the first tightness result of the Keston-Stigum reconstruction bound in roughly a decade, and later \cite{liu2018large} gave a complete answer to the question on how small the asymmetry is necessary for the tightness of the reconstruction threshold.

For non-binary models, the simplest case is the $q$-state symmetric channel which corresponds to the Potts model in statistical physics, with the following transition matrix
$$
\mathbf{M}=\left(
\begin{array}{cccccccc}
p_0 & p_1 & \cdots & p_1  \\
p_1 & p_0 & \cdots & p_1  \\
\vdots & \vdots & \ddots & \vdots \\
p_1 & p_1 & \cdots & p_0
\end{array}
\right)_{q\times q}.
$$
\cite{sly2009reconstruction} established the Kesten-Stigum bound for the $3$-state Potts model on regular trees of large degree and showed that the Kesten-Stigum bound is not tight when $q \geq 5$. Motivated by the K$80$ model (\cite{kimura1980simple}) that is one of the most classical Markov DNA evolution models, \cite{liu2018tightness} proposed the following model to distinguish between transitions and transversions, whose transition matrix has two mutation classes with $q$ states in each class
\begin{equation}
\label{k80_model}
M_{ij}= \left\{\begin{array}{ll} p_0 & \quad\textrm{if}\ i=j,
\\
p_1& \quad \textrm{if}\ i\neq j\ \textrm{and}\ i, j\ \textrm{are in
	the same category},
\\
p_2 & \quad \textrm{if}\ i\neq j\ \textrm{and}\ i, j\ \textrm{are in
	different categories}.
\end{array}
\right.
\end{equation}
When the number of states are more than or equal to $8$, 
\cite{liu2018tightness} showed that the Kesten-Stigum bound is not tight.

\subsection{Existing results with $4$ states and the importance of non-tightness}
Well known, the $2$-state and $4$-state cases give the most important reconstruction on the tree models, especially for the applications in phylogenetic reconstruction since they correspond to some of the most basic phylogenetic evolutionary models (see, for instance, the discussions in Section $2.5.1$ of \cite{mossel2016deep}). However, the $4$-state case is much more challenging and open until very few new results established recently. For the symmetric model with $4$ states, \cite{ricci2019typology} showed that in the assortative (ferromagnetic) case the Kesten-Stigum bound is always tight, while in the disassortative (antiferromagnetic) case the Kesten-Stigum bound is tight in a large degree regime and not tight in a low degree regime. Later, \cite{liu2018big} investigated a $4$-state asymmetric model whose transition matrix is of the form 
$$
\mathbf{P}=\left(
\begin{array}{@{}cc|cc@{}}
p_0&\; p_1&\; p_1'&\; p_1'
\\ 
p_1&\;  p_0&\; p_1'&\; p_1'
\\\hline
p_1&\;  p_1&\; p_0' &\; p_1'
\\
p_1&\;  p_1&\; p_1'&\; p_0'
\end{array}
\right),
$$
and gave specific conditions under which the Kesten-Stigum bound is not tight.

The stochastic block model has wide applications in statistics, machine learning, and data mining, to name a few. The connection between the reconstruction on the tree problem and the clustering problem in the setting of the stochastic block model, has been well established in recent years (see, for instance, \cite{mossel2013proof, mossel2014belief, neeman2014non, ricci2019typology}). Specifically, the technique used in handling balanced two clusters models is to transfer the problem of clustering to the reconstructability on trees. For the stochastic block model, an ``information-theoretically-solvable-but-computationally-hard" region appears, whenever the Kesten-Stigum bound is not tight for the corresponding reconstruction on the tree problem. Further information can be seen in \cite{ricci2019typology} under the name ``hybrid-hard phases".

\subsection{Motivation and main result}
While the reconstructability of the $4$-state case of the model in equation \eqref{k80_model} is still an open problem, in this paper we are able to give a rigorous answer to the reconstructible question of the $4$-state case of a more complicated and generalized model. Inspired by the $q_1+q_2$ stochastic block model proposed in \cite{ricci2019typology} (see Fig. $5$ therein for an illustration), we extend model in equation \eqref{k80_model} to incorporate different in-block transition probabilities. That is, in this paper, we focus on a $4$-state model with the transition matrix
\begin{equation}
\label{new model}
\mathbf{P}=\left(
\begin{array}{@{}cc|cc@{}}
p_0&\; p_1&\; p_2&\; p_2
\\ 
p_1&\;  p_0&\; p_2&\; p_2
\\\hline
p_2&\;  p_2&\; \overline{p}_0 &\; \overline{p}_1
\\
p_2&\;  p_2&\; \overline{p}_1&\; \overline{p}_0
\end{array}
\right).
\end{equation}
Besides different out-block transition probabilities ($p_2$) characterized in \cite{liu2018tightness}, the model under investigation has different in-block transition probabilities ($p_0$ and $p_1$ in one block, $\overline{p}_0$ and $\overline{p}_1$ in the other block). 

It is easy to see that $\mathbf{P}$ has $4$ eigenvalues: $1$, $\lambda_1=p_0-p_1$, $\lambda_2=p_0+p_1-2p_2$, and $\lambda_3=\overline{p}_0-\overline{p}_1$. Let $\lambda$ be the second largest eigenvalue by absolute value. Considering that $d|\lambda|^2 > 1$ always implies reconstruction, we only investigate $d|\lambda|^2 \leq 1$ in the following context. Our main result is the following theorem, whose rigorous proof is given in Section \ref{Sec:Proof_of_Main_Theorem}.
\begin{Main Theorem}
	\label{reconstruction} If $|\lambda_1|\neq|\lambda_3|$ and $0<|\lambda_2|< \max\left\{|\lambda_1|, |\lambda_3|\right\}$, the Kesten-Stigum bound is not tight for every $d$, i.e. the
	reconstruction is solvable for some $\lambda$ even if
	$d\lambda^2<1$.
\end{Main Theorem}

Since $\lambda_1$ and $\lambda_3$ play symmetric roles in this symmetric model~\eqref{new model}, without loss of generality, we presume $|\lambda_1|>|\lambda_3|$ in the sequel. 

\subsection{Structure of the paper and proof sketch}
 The technique used here was initiated in \cite{chayes1986mean} in the context of spin glasses. In Section \ref{Sec:Preliminary_Results}, we give detailed definitions and interpretations, conduct preliminary analyses, and then provide an equivalent condition for non-reconstruction: 
$$
\lim_{n\to \infty}x_n=\lim_{n\to \infty}\overline{x}_{n}=0.
$$
Here, $x_n$ and $\overline{x}_{n}$ represent the probabilities of giving a correct guess of the root given the spins $\sigma(n)$ at distance $n$ from the root minus the probability of guessing the root randomly which is $1/4$ in this case, for the root being in block $1$ and block $2$ respectively. Nonreconstruction means that the mutual information between the root and the spins at distance $n$ goes to $0$ as $n$ tends to infinity, therefore one standard to classify reconstruction and nonreconstruction is to analyze the quantity   
$x_n$ while in this paper we also need to consider the limiting behavior of $\overline{x}_{n}$.

In Section \ref{Sec:Distributional_Recursion}, after in-depth investigation of the recursive relationship, we develop a two dimensional dynamical system of the linear diagonal canonical form regarding quantities $x_{n+1}$ and $\overline{z}_{n+1}$ through two new variables $\mathcal{X}_n=x_n+\overline{z}_n$ and $\mathcal{Z}_n=-\overline{z}_n$:
\begin{equation*}
\left\{\begin{array}{ll}
\mathcal{X}_{n+1}=d\lambda_1^2\mathcal{X}_n+\frac{d(d-1)}{2}\left(-4\lambda_1^4\mathcal{X}_n^2+8\lambda_1^2\lambda_2^2\mathcal{X}_n\mathcal{Z}_n\right)
+R_x+R_z+V_x
\\
\
\\
\mathcal{Z}_{n+1}=d\lambda_2^2\mathcal{Z}_n+\frac{d(d-1)}{2}\left[\lambda_1^4\mathcal{X}_n^2-8\lambda_2^4\mathcal{Z}_n^2+\frac14\lambda_3^4(\overline{x}_n-\overline{y}_n)^2\right]
-R_z+V_z.
\end{array}
\right.
\end{equation*}
Here, $\overline{z}_{n}$ represents the opposite case of $x_n$ as giving a wrong guess in another block. By symmetry, we can also obtain the dynamical system involving $\overline{x}_n$ simply through replacing $\lambda_1$ by $\lambda_3$.
In Section \ref{Sec:Concentration_Analysis}, we show that $R_x$, $R_z$, $V_x$, and $V_z$ are just small perturbations in the above dynamical system in order to study its stability, ensure that the decrease from $x_n$ to $x_{n+1}$ is never too large to lose construction,
and establish crucial concentration results, by fully taking advantage of the Markov random field property and the symmetries in the probability transition matrix and the network structure. In Section \ref{Sec:Proof_of_Main_Theorem}, by means of the method of reductio ad absurdum, we show that $x_n$ and $\overline{x}_{n}$ can not simultaneously converge to zero as $n$ goes to $\infty$, and then establish the nontightness of Kesten-Stigum bound.

\section{Preparation}
\label{Sec:Preliminary_Results}
\subsection{Notations}
\label{Sec:Notations}
Let $u_1,\ldots,u_d$
be the children of the root $\rho$ and $\mathbb{T}_v$ be the subtree of
descendants of $v\in \mathbb{V}$. Denote the
$n$th level of the tree by $L_n=\{v\in\mathbb{V}: d(\rho, v)=n\}$ with $d(\cdot,
\cdot)$ being the graph distance on $\mathbb{T}$.
Denote $\sigma(n)$ as the spins on $L_n$, $\sigma^i(n)$ as $\sigma(n)$ conditioned on $\sigma_\rho = i$, and $\sigma_j(n)$ as the spins on $L_n\cap \mathbb{T}_{u_j}$ where $u_j$ is one of the children of the root $\rho$. For the notations involving $\sigma(n)$ in the sequel, we consistently use superscript to denote the conditional on a specific configuration of the root, and use the subscript to denote the conditional on a specific offspring of the root. 

For a configuration $A$ on the spins of $L_n$, define the posterior function by
$$
f_n(i, A)=\mathbf{P}(\sigma_\rho=i\mid\sigma(n)=A)=\mathbf{P}(\sigma_{u_j}=i\mid\sigma_j(n+1)=A), 
$$
for $i=1, 2, 3, 4$ and $j=1,\cdots,d$,
where the second equality holds by the recursive nature of the tree. 
Define $X_i(n)$ as the posterior probability that the root $\rho$ is taking the configuration $i$ given the random configuration $\sigma(n)$ on the spins in $L_n$, i.e., 
$$
X_i(n)=f_n(i, \sigma(n)), \quad i=1, 2, 3, 4.
$$
Apparently one has
$$
X_1(n)+X_2(n)+X_3(n)+X_4(n)=1.
$$
By the block characteristic of the model, we know that regarding the first (resp. second) block, $X_1(n)$ and $X_2(n)$ (resp. $X_3(n)$ and $X_4(n)$) have the same distribution. 
Considering that the stationary distribution $\pi=(\pi_1, \pi_2, \pi_3, \pi_4)$ of
$\mathbf{P}$ is given by
$$
\pi_1=\pi_2=
\pi_3=\pi_4=\frac{1}{4},
$$
we further have
$$
\mathbf{E}(X_1(n))=\mathbf{E}(X_2(n))= \mathbf{E}(X_3(n))=\mathbf{E}(X_4(n))=\frac{1}4.
$$

From the symmetry and the block characteristic of the model, we know that
$$f_n(i, \sigma^j(n))=f_n(j, \sigma^i(n)),\quad \quad\text{for}\quad i\neq j, \quad i,j \in\{1,2\} \text{ or } \{3,4\},$$
and 
$$f_n(1,\sigma^3(n))=f_n(1,\sigma^4(n)).$$
Define $Y_{ij}(n)$ as the posterior probability that $\sigma_{u_j}=i$ given the random configuration $\sigma^1_j (n + 1)$ on spins in 
$L(n+1) \cap \mathbb{T}_{u_j}$, i.e.,
$$
Y_{ij}(n)=f_n(i, \sigma_j^1(n+1)), \quad \quad\text{for}\quad i=1, 2, 3, 4, \quad j=1,\cdots,d,
$$
where the random variables $\{Y_{ij}(n)\}$ are independent and identically distributed and satisfy
$$Y_{1j}(n)+Y_{2j}(n)+Y_{3j}(n)+Y_{4j}(n)=1.$$

We define the following moment variables to analyze the differences between different inferences of $\sigma_\rho$ given the spins $\sigma(n)$ at distance $n$ from the root $\rho$ and the probability of guessing the root randomly:

$$x_{n}=\mathbf{E}\left(f_n(1, \sigma^1(n))-\frac{1}{4}\right), 
\quad\quad y_{n}=\mathbf{E}\left(f_n(2, \sigma^1(n))-\frac{1}{4}\right), $$
$$z_{n}=\mathbf{E}\left(f_n(1, \sigma^3(n))-\frac{1}{4}\right),
\quad\quad u_{n}=\mathbf{E}\left(f_n(1, \sigma^1(n))-\frac{1}{4}\right)^2,$$
$$v_{n}=\mathbf{E}\left(f_n(2, \sigma^1(n))-\frac{1}{4}\right)^2,
\quad\quad w_{n}=\mathbf{E}\left(f_n(1, \sigma^3(n))-\frac{1}{4}\right)^2,$$
$$\overline{x}_{n}=\mathbf{E}\left(f_n(3, \sigma^3(n))-\frac{1}{4}\right), 
\quad\quad \overline{y}_{n}=\mathbf{E}\left(f_n(4, \sigma^3(n))-\frac{1}{4}\right), $$
$$\overline{z}_{n}=\mathbf{E}\left(f_n(3, \sigma^1(n))-\frac{1}{4}\right),
\quad\quad \overline{u}_{n}=\mathbf{E}\left(f_n(3, \sigma^3(n))-\frac{1}{4}\right)^2,$$
$$\overline{v}_{n}=\mathbf{E}\left(f_n(4, \sigma^3(n))-\frac{1}{4}\right)^2,
\quad\quad \overline{w}_{n}=\mathbf{E}\left(f_n(3, \sigma^1(n))-\frac{1}{4}\right)^2.$$

\subsection{Preliminary analyses}
We firstly establish some important lemmas which will be used frequently in the sequel.
\begin{lemma}
	\label{lemma1}
	For any $n\in \mathbb{N}\cup\{0\}$, we have
	\begin{enumerate}[(a)]
		\item $\displaystyle x_n=4\mathbf{E}\left(X_1(n)-\frac14\right)^2=u_{n}+v_{n}+2w_{n}\geq 0$.
		
		\item $\displaystyle -\frac{x_n+y_n}2=z_n=\overline{z}_n=-\frac{\overline{x}_n+\overline{y}_n}{2}\leq 0$.
		
		\item $\displaystyle x_n+z_n\geq 0, \quad \overline{x}_{n}+z_n\geq 0$.
	\end{enumerate}
\end{lemma}
\begin{proof} 
	\begin{enumerate}[(a)]
		\item By the law of total probability and Bayes' theorem, we have
		\begin{eqnarray*}
			\mathbf{E}f_n(1, \sigma^1(n))
			&=&\sum_A f_n(1,A)\mathbf{P}(\sigma(n)=A\mid\sigma_\rho=1)
			\\
			&=&4\sum_A f_n(1,A) \mathbf{P}(\sigma_\rho=1\mid\sigma(n)=A)\mathbf{P}(\sigma(n)=A)
			\\
			&=&4\sum_Af_n^2(1,A)\mathbf{P}(\sigma(n)=A)
			=4\mathbf{E} (X_1(n))^2.
		\end{eqnarray*}
		Recall that $x_{n}$ is defined as $x_{n}=\mathbf{E}\left(f_n(1, \sigma^1(n))-\frac{1}{4}\right)$, and then by the fact that $\mathbf{E}(X_1(n))=\frac14$ we have
		$$
		x_n=4\left(\mathbf{E} (X_1(n))^2-\left(\frac{1}{4}\right)^2\right)=4\mathbf{E}\left(X_1(n)-\frac14\right)^2.
		$$
		Furthermore, by the law of total expectation, we have
		\begin{eqnarray*}
			x_n&=&4\mathbf{E}\left(X_1(n)-\frac14\right)^2
			\\
			&=&4\sum_{i=1}^4\mathbf{E}\left(\left(X_1(n)-\frac14\right)^2\bigg | \sigma_\rho=i\right)\mathbf{P}(\sigma_\rho=i)
			\\
			&=&4\left[\mathbf{P}(\sigma_\rho=1)\mathbf{E}\left(f_n(1, \sigma^1(n))-\frac14\right)^2+\mathbf{P}(\sigma_\rho=2)\mathbf{E}\left(f_n(1, \sigma^2(n))-\frac14\right)^2\right.
			\\
			&&\quad+\left.\mathbf{P}(\sigma_\rho=3)\mathbf{E}\left(f_n(1, \sigma^3(n))-\frac14\right)^2
			+\mathbf{P}(\sigma_\rho=4)\mathbf{E}\left(f_n(1, \sigma^4(n))-\frac14\right)^2\right]
			\\
			&=&u_{n}+v_{n}+2w_{n}.
		\end{eqnarray*}

	\item 	Similarly, we have 
\begin{equation}
\label{zn}
z_n=4\mathbf{E}\left(X_1(n)X_3(n)\right)-\frac14
=\mathbf{E}\left(f_n(1, \sigma^3(n))-\frac14\right)=\overline{z}_n,
\end{equation} 

\begin{equation*}
y_n+\frac14
=\sum_A f_n(2,A)\mathbf{P}(\sigma(n)=A\mid\sigma_\rho=1)
=4\mathbf{E}\left(X_1(n)X_2(n)\right),
\end{equation*}
and then
\begin{equation}
\label{y}
y_n=4\mathbf{E}\left(X_1(n)-\frac14\right)\left(X_2(n)-\frac14\right).
\end{equation}
It follows from the Cauchy-Schwarz inequality that
$$
\left[\mathbf{E}\left(X_1(n)-\frac14\right)\left(X_2(n)-\frac14\right)\right]^2\leq\mathbf{E}\left(X_1(n)-\frac14\right)^2\mathbf{E}\left(X_2(n)-\frac14\right)^2,
$$
which implies 
\begin{equation}
\label{cauchy}
\left(\frac14y_n\right)^2\leq\left(\frac14x_n\right)^2, \quad\textup{i.e.}\quad -x_n\leq y_n\leq x_n.
\end{equation}
By the definitions of $x_n$, $y_n$ and $z_n$, 
we know that
$z_n=-\frac{x_n+y_n}2$, and thus equation \eqref{cauchy} implies $z_n\leq 0$.

\item An analogous proof of $$x_n+z_n=x_n-\frac{x_n+y_n}{2}=\frac{x_n-y_n}{2}\geq 0\quad \text{and} \quad \overline{x}_{n}+z_n\geq 0$$
can be easily carried out. 	  			
	\end{enumerate}
\end{proof}

\begin{lemma}
	\label{xnun} For any $n\in \mathbb{N}\cup\{0\}$, we have
	\begin{enumerate}[(a)]
		
		\item $\displaystyle \mathbf{E}\left(f_n(1, \sigma^1(n))-\frac{1}{4}\right)\left(f_n(2, \sigma^1(n))-\frac14\right)=\frac14y_n+\left(v_{n}-\frac14x_n\right)$.
		
		\item $
		\!
		\begin{aligned}[t]
		&\displaystyle \mathbf{E}\left(f_n(1, \sigma^1(n))-\frac14\right)\left(f_n(3,\sigma^1(n))-\frac14\right)\\
		=&\frac14z_n-\frac12\left(u_{n}-\frac14x_n\right)-\frac12\left(v_{n}-\frac14x_n\right).
		\end{aligned}
		$
		
		\item $\displaystyle \mathbf{E}\left(f_n(2, \sigma^1(n))-\frac14\right)\left(f_n(3, \sigma^1(n))-\frac14\right)=\frac14z_n-\left(v_{n}-\frac14x_n\right)$.

		\item $
		\!
		\begin{aligned}[t]
		&\displaystyle \mathbf{E}\left(f_n(3, \sigma^1(n))-\frac14\right)\left(f_n(4, \sigma^1(n))-\frac14\right)\\
		=&\frac14\overline{y}_{n}+\frac12\left(u_{n}-\frac14x_n\right)+\frac32\left(v_{n}-\frac14x_n\right)-\left(\overline{w}_{n}-\frac14\overline{x}_{n}\right).
		\end{aligned}
		$
		
		\item $\displaystyle\mathbf{E}\left(f_n(1, \sigma^3(n))-\frac14\right)\left(f_n(2, \sigma^3(n))-\frac14\right)=\frac14y_n-\left(v_{n}-\frac14x_n\right)$.
	\end{enumerate}
\end{lemma}
\begin{proof}
	We only prove (a) and (b) and the others can be shown analogously.
	\begin{enumerate}[(a)]
		\item By the law of total probability, one has
		\begin{equation*}
		\begin{aligned}
		&\mathbf{E}\bigg(f_n(1,\sigma^1(n))f_n(2,\sigma^1(n))\bigg)
		\\
		=&\sum_{A}\mathbf{P}(\sigma_\rho=1\mid\sigma(n)=A)\mathbf{P}(\sigma_\rho=2\mid\sigma(n)=A)\mathbf{P}(\sigma(n)=A\mid\sigma_\rho=1)
		\\
		=&\sum_{A}\left[\mathbf{P}(\sigma_\rho=2\mid\sigma(n)=A)\right]^2\mathbf{P}(\sigma(n)=A\mid\sigma_\rho=1)
		\\
		=&\mathbf{E}\left(f_n(2,\sigma^1(n))\right)^2,
		\end{aligned}
		\end{equation*}
		therefore
		\begin{equation*}
		\begin{aligned}
		&\mathbf{E}\left(f_n(1,\sigma^1(n))-\frac14\right)\left(f_n(2,\sigma^1(n))-\frac14\right)\\
		=&v_{n}+\frac14\left(y_n-x_n\right)=\frac14y_n+\left(v_{n}-\frac14x_n\right).
		\end{aligned}
		\end{equation*}
		\item By the fact that $f_n(3, \sigma^1(n))$ and $f_n(4, \sigma^1(n))$ have the same distribution, and the equation that
		$$f_n(1, \sigma^1(n))+f_n(2, \sigma^1(n))+f_n(3, \sigma^1(n))+f_n(4, \sigma^1(n))=1,$$
		plugging in the result of (a), we can obtain that
		\begin{equation*}
		\begin{aligned}
		&\mathbf{E}\left(f_n(1, \sigma^1(n))-\frac14\right)\left(f_n(3, \sigma^1(n))-\frac14\right)\\
		=&\frac14z_n-\frac12\left(u_{n}-\frac14x_n\right)-\frac12\left(v_{n}-\frac14x_n\right),
		\end{aligned}
		\end{equation*}

	\end{enumerate}
as desired.
\end{proof}

Recall that $Y_{ij}(n)$ is defined as the posterior probability that $\sigma_{u_j}=i$ given the random configuration $\sigma^1_j (n + 1)$ on spins in 
$L(n+1) \cap \mathbb{T}_{u_j}$, i.e., $Y_{ij}(n)=f_n(i, \sigma_j^1(n+1))$, for $i \in \{1, 2, 3, 4\}$ and $j\in \{1,\cdots,d\}$.
The random vectors $(Y_{ij}(n))_{i=1}^{4}$ are independent by the symmetry of the model, and its central moments are investigated in the following lemma.

\begin{lemma}
	\label{lemma:Yproperties} For each $1\leq j\leq d$, we have
	\begin{enumerate}[(a)]
		\item $\!
		\begin{aligned}[t]\mathbf{E}\left(Y_{1j}(n)-\frac14\right)=\lambda_1 x_n+(\lambda_1-\lambda_2)z_n\end{aligned}$.
		
		\item $\!
		\begin{aligned}[t]\mathbf{E}\left(Y_{2j}(n)-\frac14\right)=-\lambda_1 x_n-(\lambda_1+\lambda_2)z_n\end{aligned}$.
		
		\item $\!
		\begin{aligned}[t]\mathbf{E}\left(Y_{ij}(n)-\frac14\right)=\lambda_2 z_n, \quad i=3,4.
		\end{aligned}$
		
		\item $\!
		\begin{aligned}[t]
		\mathbf{E}\left(Y_{1j}(n)-\frac14\right)^2
		=\frac{1}{4}x_n+\lambda_1\left(u_{n}-\frac14x_n\right)+(\lambda_1-\lambda_2)\left(w_{n}-\frac14x_n\right).
		\end{aligned}
		$

		\item $
		\!
		\begin{aligned}[t]
		\mathbf{E}\left(Y_{2j}(n)-\frac14\right)^2
		=\frac14x_n-\lambda_1\left(u_{n}-\frac14x_n\right)-(\lambda_1+\lambda_2)\left(w_{n}-\frac14x_n\right).
		\end{aligned}
		$
		
		\item 
		$\!
		\begin{aligned}[t]\mathbf{E}\left(Y_{ij}(n)-\frac{1}{4}\right)^2
		=\frac14\overline{x}_{n}+\lambda_2\left(\overline{w}_{n}-\frac14\overline{x}_{n}\right), \quad i=3,4.
		\end{aligned}
		$
		
		\item $
		\!
		\begin{aligned}[t]
		\mathbf{E}\left(Y_{1j}(n)-\frac14\right)\left(Y_{2j}(n)-\frac14\right)
		=\frac14y_n+\lambda_2\left(v_{n}-\frac14x_n\right).
		\end{aligned}
		$
		
		\item 
		$\!
		\begin{aligned}[t]
		&\mathbf{E}\left(Y_{1j}(n)-\frac14\right)\left(Y_{ij}(n)-\frac14\right)\\
		=&\frac{z_n}{4}+\frac{\lambda_1-\lambda_2}2\left(v_{n}-\frac14x_n\right)+\frac{\lambda_1+\lambda_2}2\left(w_{n}-\frac14x_n\right), \quad i=3,4.
		\end{aligned}$
		
		\item
		$\!
		\begin{aligned}[t] 
		&\mathbf{E}\left(Y_{2j}(n)-\frac{1}{4}\right)\left(Y_{ij}(n)-\frac14\right)\\
		=&\frac{z_n}{4}-\frac{\lambda_1+\lambda_2}2\left(v_{n}-\frac14x_n\right)-\frac{\lambda_1-\lambda_2}2\left(w_{n}-\frac14x_n\right), \quad i=3,4.
		\end{aligned}
		$
		
		\item $
		\!
		\begin{aligned}[t]
		\mathbf{E}\left(Y_{3j}(n)-\frac14\right)\left(Y_{4j}(n)-\frac14\right)
		=&\frac14\overline{y}_{n}-\lambda_2\left(\overline{v}_{n}-\frac14\overline{x}_{n}\right).
		\end{aligned}
		$
	\end{enumerate}
\end{lemma}

\begin{proof} We only prove (a), (b), and (c) and the others can be shown analogously.
	\begin{enumerate}[(a)]
		\item Conditioning on $\sigma_{u_j}=i$ for $i \in \{1, 2, 3, 4\}$, we have
		\begin{equation*}
		\begin{aligned}[t]
		\mathbf{E}\left(Y_{1j}(n)-\frac14\right)
		=&p_{11}\mathbf{E}\left(f_n(1,\sigma^1(n))-\frac14\right)+p_{12}\mathbf{E}\left(f_n(1, \sigma^2(n))-\frac14\right)\\
		&+p_{13}\mathbf{E}\left(f_n(1, \sigma^3(n))-\frac14\right)
		+p_{14}\mathbf{E}\left(f_n(1, \sigma^4(n))-\frac14\right)
		\\
		=&\left(p_0-p_1\right)x_n+2(p_2-p_1)z_n
		\\
		=&\lambda_1x_n+(\lambda_1-\lambda_2)z_n.
		\end{aligned}
		\end{equation*}	
		
		\item Similar, we can obtain
		\begin{equation*}
		\begin{aligned}[t]
		\mathbf{E}\left(Y_{2j}(n)-\frac14\right)
		=&\left(p_1-p_0\right)x_n+2(p_2-p_0)z_n
		\\
		=&-\lambda_1x_n-(\lambda_1+\lambda_2)z_n.
		\end{aligned}
		\end{equation*}	
		
		\item
		It follows immediately from the identity $\sum_{i=1}^{4}Y_{ij}(n)=1$ that, for $i=3,4,$
		$$
		\mathbf{E}\left(Y_{ij}(n)-\frac14\right)=-\frac{1}{2}\sum_{i=1}^2\mathbf{E}\left(Y_{ij}(n)-\frac14\right)=\lambda_2 z_n.
		$$

	\end{enumerate}
\end{proof}

\subsection{An equivalent condition for non-reconstruction}
\label{sec3}\noindent 

If the reconstruction problem is solvable, $\sigma(n)$ contains
significant information of the root variable. This can be expressed
in several equivalent ways (see \cite{mossel2001reconstruction, mossel2004survey}). 
\begin{lemma}
	\label{equivalent} The non-reconstruction is equivalent to
	$$
	\lim_{n\to \infty}x_n=\lim_{n\to \infty}\overline{x}_{n}=0.
	$$
\end{lemma}


\section{Recursive formulas}
\label{Sec:Distributional_Recursion}

\subsection{Distributional recursion}
\label{Zresults}
Consider $A$ as a
configuration on $L(n+1)$, and let $A_j (j=1,\cdots, d)$ be its restriction to
$\mathbb{T}_{u_j}\bigcap L(n+1)$ where $u_j$
is the $j$th child of the root $\rho$. Then from the Markov random field
property, we have
\begin{equation}
\label{recursion} f_{n+1}(1,A)=\frac{N_1(n)}{N_1(n)+N_2(n)+N_3(n)+N_4(n)},
\end{equation}
where $N_k(n)$ is given by
$$N_k(n)=\prod_{j=1}^d\left[\sum_{i=1}^4p_{ki}\mathbf{P}(\sigma_j(n+1)=A_j\mid\sigma_{u_j}=i)\right],\quad k\in\{1,2,3,4\}.$$
Recall that $Y_{ij}(n)=f_n(i, \sigma_j^1(n+1))$. Setting $A=\sigma^1(n+1)$, we have
\begin{equation}
\label{eq:f_Z}
f_{n+1}(1,\sigma^1(n+1))=\frac{Z_1(n)}{Z_1(n)+Z_2(n)+Z_3(n)+Z_4(n)},
\end{equation}
where
$$\displaystyle
Z_i(n)=\left\{\begin{array}{ll} \prod_{j=1}^d\left[1+2(\lambda_1+\lambda_2)\left(Y_{1j}(n)-\frac14\right)-2(\lambda_1-\lambda_2)\left(Y_{2j}(n)-\frac14\right)\right] \quad i=1
\\
\prod_{j=1}^d\left[1-2(\lambda_1-\lambda_2)\left(Y_{1j}(n)-\frac14\right)+2(\lambda_1+\lambda_2)\left(Y_{2j}(n)-\frac14\right)\right]\quad i=2
\\
\prod_{j=1}^d\left[1+2(\lambda_2+\lambda_3)\left(Y_{3j}(n)-\frac14\right)+2(\lambda_2-\lambda_3)\left(Y_{4j}(n)-\frac14\right)\right] \quad i=3
\\
\prod_{j=1}^d\left[1+2(\lambda_2-\lambda_3)\left(Y_{3j}(n)-\frac14\right)+2(\lambda_2+\lambda_3)\left(Y_{4j}(n)-\frac14\right)\right] \quad i=4,
\end{array}
\right.
$$
i.e., $Z_i(n)=\frac{N_i(n)}{\prod_{j=1}^d \mathbf{P}(\sigma_j(n+1)=A_j)}.$

\begin{lemma}
	\label{Z1Z2} For any nonnegative $n\in \mathbb{Z}^+$, we have
	$$
	\mathbf{E}\left(Z_1(n)Z_2(n)\right)=\mathbf{E}Z_2^2(n).
	$$
\end{lemma}

\begin{proof}
	For any configuration $A=(A_1, \ldots, A_d)$ with
	$A_j$ denoting the spins on $L_{n+1}\cap \mathbb{T}_{u_j}$, we have
	\begin{eqnarray*}
		Z_i(n)=4\frac{\mathbf{P}(\sigma(n+1)=A)}{\prod_{j=1}^d\mathbf{P}(\sigma_j(n+1)=A_j)}\mathbf{P}(\sigma_\rho=i\mid\sigma(n+1)=A), \quad \text{for}\; i=1, 2.
	\end{eqnarray*}
	By the symmetry of the tree, we have
	\begin{equation*}
	\begin{aligned}
	\mathbf{E}\left(Z_1(n)Z_2(n)\right)
	=&16\sum_A\left(\frac{\mathbf{P}(\sigma(n+1)=A)}{\prod_{j=1}^d\mathbf{P}(\sigma_j(n+1)=A_j)}\right)^2\mathbf{P}(\sigma_\rho=1\mid\sigma(n+1)=A)\\
	&\times \mathbf{P}(\sigma_\rho=2\mid\sigma(n+1)=A)
	\mathbf{P}(\sigma(n+1)=A\mid\sigma_\rho=1)
	\\
	=&16\sum_A\left(\frac{\mathbf{P}(\sigma(n+1)=A)}{\prod_{j=1}^d\mathbf{P}(\sigma_j(n+1)=A_j)}\right)^2\mathbf{P}^2(\sigma_\rho=2\mid\sigma(n+1)=A)\\
	&\times \mathbf{P}(\sigma(n+1)=A\mid\sigma_\rho=1)
	\\
	=&\mathbf{E}Z_2^2(n),
	\end{aligned}
	\end{equation*}
	as desired.
\end{proof}

By Lemma \ref{lemma:Yproperties}, the means and variances of monomials of $Z_i(n)$ can be approximated as follows:
\begin{lemma} 
	\label{lemma:Z}	
	One has
	\begin{enumerate}[(i)]
		\item $
		\!
		\begin{aligned}[t]
		\mathbf{E}Z_1(n)
		=&1+ d\lambda_1^24(x_n+z_n)-d\lambda_2^24z_n\\
		&+\frac{d(d-1)}{2}\left[4\lambda_1^2(x_n+z_n)-4\lambda_2^2z_n\right]^2+O (x_n^3).
		\end{aligned}
		$
		
		\item $
		\!
		\begin{aligned}[t]
		\mathbf{E}Z_2(n)
		=&1-d\lambda_1^24(x_n+z_n)-d\lambda_2^24z_n\\
		&+\frac{d(d-1)}{2}\left[4\lambda_1^2(x_n+z_n)+4\lambda_2^2z_n\right]^2+O(x_n^3).
		\end{aligned}
		$
		
		\item $
		\!
		\begin{aligned}[t]
		\mathbf{E}Z_i(n)
		=1+ d\lambda_2^24z_n+\frac{d(d-1)}{2}\left(4\lambda_2^2z_n\right)^2+O(x_n^3), \quad i=3,4.
		\end{aligned} $
		
		\item 
		$		\!
		\begin{aligned}[t]
		\mathbf{E}Z_1^2(n)=1+d\Pi_1+\frac{d(d-1)}{2}\Pi_1^2+O(x_n^3),
		\end{aligned}
		$
		where
		$$\begin{aligned}
			\Pi_1=&\mathbf{E}\left[1+2(\lambda_1+\lambda_2)\left(Y_{1j}(n)-\frac14\right)-2(\lambda_1-\lambda_2)\left(Y_{2j}(n)-\frac14\right)\right]^2-1
			\\
			=&12\lambda_1^2(x_n+z_n)-12\lambda_2^2z_n+16\lambda_1^2\lambda_2\left(u_{n}-\frac14x_n\right)\\
			&-8(\lambda_1^2-\lambda_2^2)\lambda_2\left(v_n-\frac14x_n\right)+8(\lambda_1^2-\lambda_2^2)\lambda_2\left(w_n-\frac14x_n\right).
		\end{aligned}$$
		
		\item 
		$\!
		\begin{aligned}[t]
		\mathbf{E}Z_2^2(n)=\mathbf{E}Z_1(n)Z_2(n)=1+d\Pi_2+\frac{d(d-1)}{2}\Pi_2^2+O(x_n^3),
		\end{aligned}
		$
		where
		$$\begin{aligned}
			\Pi_2=&\mathbf{E}\left[1-2(\lambda_1-\lambda_2)\left(Y_{1j}(n)-\frac14\right)+2(\lambda_1+\lambda_2)\left(Y_{2j}(n)-\frac14\right)\right]^2-1
			\\
			=&-4\lambda_1^2(x_n+z_n)-12\lambda_2^2z_n-16\lambda_1^2\lambda_2\left(u_{n}-\frac14x_n\right)\\
			&-8(\lambda_1^2-\lambda_2^2)\lambda_2\left(v_n-\frac14x_n\right)-8(3\lambda_1^2+\lambda_2^2)\lambda_2\left(w_n-\frac14x_n\right).
		\end{aligned}$$
		
		\item 
		$\!
		\begin{aligned}[t]
		\mathbf{E}Z_i^2(n)=1+d\Pi_3+\frac{d(d-1)}{2}\Pi_3^2+O(x_n^3), 
		\end{aligned}
		$
		for $i=3,4$, where
		\begin{eqnarray*}
			\Pi_3&=&\mathbf{E}\left[1+2(\lambda_2+\lambda_3)\left(Y_{3j}(n)-\frac14\right)+2(\lambda_2-\lambda_3)\left(Y_{4j}(n)-\frac14\right)\right]^2-1
			\\
			&=&4\lambda_2^2z_n+2\lambda_3^2 (\overline{x}_{n}-\overline{y}_n)-8(\lambda_2^2-\lambda_3^2)\lambda_2\left(\overline{v}_{n}-\frac14\overline{x}_{n}\right)+8(\lambda_2^2+\lambda_3^2)\lambda_2\left(\overline{w}_{n}-\frac14\overline{x}_{n}\right).
		\end{eqnarray*}
		
		\item $\!
		\begin{aligned}[t]
		\mathbf{E}Z_{1}(n)Z_{i}(n)
		=1+d\Pi_4+\frac{d(d-1)}{2}\Pi_4^2+O(x_n^3),
		\end{aligned}
		$
		for $i=3,4$, where
		\begin{eqnarray*}
			\Pi_4&=&\mathbf{E}\left[1+2(\lambda_1+\lambda_2)\left(Y_{1j}(n)-\frac14\right)-2(\lambda_1-\lambda_2)\left(Y_{2j}(n)-\frac14\right)\right]\\
			&&\times \left[1+2(\lambda_2+\lambda_3)\left(Y_{3j}(n)-\frac14\right)+2(\lambda_2-\lambda_3)\left(Y_{4j}(n)-\frac14\right)\right]-1
			\\
			&=&4\lambda_1^2(x_n+z_n)+4\lambda_2^2z_n+8(\lambda_1^2-\lambda_2^2)\lambda_2\left(v_n-\frac14x_n\right)+8(\lambda_1^2+\lambda_2^2)\lambda_2\left(w_n-\frac14x_n\right).
		\end{eqnarray*}
		
		\item $\!
		\begin{aligned}[t]
		\mathbf{E}Z_2(n)Z_i(n)
		=1+d\Pi_5+\frac{d(d-1)}{2}\Pi_5^2+O(x_n^3)
		\end{aligned}
		$,
		for $i=3,4$, where
		\begin{eqnarray*}
			\Pi_5&=&\mathbf{E}\left[1-2(\lambda_1-\lambda_2)\left(Y_{1j}(n)-\frac14\right)+2(\lambda_1+\lambda_2)\left(Y_{2j}(n)-\frac14\right)\right]\\
			&&\times \left[1+2(\lambda_2+\lambda_3)\left(Y_{3j}(n)-\frac14\right)+2(\lambda_2-\lambda_3)\left(Y_{4j}(n)-\frac14\right)\right]-1
			\\
			&=&
			-4\lambda_1^2(x_n+z_n)+4\lambda_2^2z_n-8(\lambda_1^2+\lambda_2^2)\lambda_2\left(v_n-\frac14x_n\right)-8(\lambda_1^2-\lambda_2^2)\lambda_2\left(w_n-\frac14x_n\right).
		\end{eqnarray*}
		
		\item $\!
		\begin{aligned}[t]
		\mathbf{E}Z_3(n)Z_4(n)
		=1+d\Pi_6+\frac{d(d-1)}{2}\Pi_6^2+O(x_n^3),
		\end{aligned}
		$
		where
		\begin{eqnarray*}
			\Pi_6&=&\mathbf{E}\left[1+2(\lambda_2+\lambda_3)\left(Y_{3j}(n)-\frac14\right)+2(\lambda_2-\lambda_3)\left(Y_{4j}(n)-\frac14\right)\right]\\
			&&\times\left[1+2(\lambda_2-\lambda_3)\left(Y_{3j}(n)-\frac14\right)+2(\lambda_2+\lambda_3)\left(Y_{4j}(n)-\frac14\right)\right]-1
			\\
			&=&-4\lambda_3^2(\overline{x}_n+\overline{z}_n)+4\lambda_2^2\overline{z}_n-8(\lambda_2^2+\lambda_3^2)\lambda_2\left(\overline{v}_n-\frac14\overline{x}_n\right)+8(\lambda_2^2-\lambda_3^2)\lambda_2\left(\overline{w}_n-\frac14\overline{x}_n\right).
		\end{eqnarray*}
	\end{enumerate}
\end{lemma}

\subsection{Main expansions of $x_{n+1}$ and $\overline{z}_{n+1}$}
In this section, we investigate the second order
recursive relations associated with $x_{n+1}$ and $\overline{z}_{n+1}$, with the assistance of the
following identity
\begin{equation}
\label{identity}
\frac{a}{s+r}=\frac{a}{s}-\frac{ar}{s^2}+\frac{r^2}{s^2}\frac{a}{s+r}.
\end{equation}
Plugging $a=Z_1(n)$, $r=Z_1(n)+Z_2(n)+Z_3(n)+Z_4(n)-1$, and $s=1$ into equation \eqref{identity}, by the definition of $x_{n}$ and equation \eqref{eq:f_Z}, we have
\begin{equation}
\begin{split}
\label{xexpansion}
&x_{n+1}+\frac14\\
=&\mathbf{E}\frac{Z_1(n)}{Z_1(n)+Z_2(n)+Z_3(n)+Z_4(n)}
\\
=&\mathbf{E}Z_1(n)-\mathbf{E}Z_1(n)\left(Z_1(n)+Z_2(n)+Z_3(n)+Z_4(n)-1\right)\\
&+\mathbf{E}\left(Z_1(n)+Z_2(n)+Z_3(n)+Z_4(n)-1\right)^2\frac{Z_1(n)}{Z_1(n)+Z_2(n)+Z_3(n)+Z_4(n)}.
\end{split}
\end{equation}
Next, plugging $a=Z_3(n)$, $r=Z_1(n)+Z_2(n)+Z_3(n)+Z_4(n)-1$, and $s=1$ in equation \eqref{identity}, by the definition of $\overline{z}_{n}$ and an analogous derivation as equation \eqref{eq:f_Z}, we can obtain
\begin{equation}
\begin{split}
\label{zexpansion}
&\overline{z}_{n+1}+\frac14\\
=&\mathbf{E}Z_3(n)-\mathbf{E}Z_3(n)\left(Z_1(n)+Z_2(n)+Z_3(n)+Z_4(n)-1\right)\\
&+\mathbf{E}\left(Z_1(n)+Z_2(n)+Z_3(n)+Z_4(n)-1\right)^2\frac{Z_3(n)}{Z_1(n)+Z_2(n)+Z_3(n)+Z_4(n)}.
\end{split}
\end{equation}
Finally, plugging the results of Section \ref{Zresults}
into equation \eqref{xexpansion} and equation \eqref{zexpansion}, and then taking
substitutions of 
$$\mathcal{X}_n=x_n+\overline{z}_n \quad\text{and}\quad \mathcal{Z}_n=-\overline{z}_n,$$ 
we obtain a two-dimensional recursive formula of the linear diagonal canonical form:
\begin{equation}
\label{eq:Z_dynamics}
\left\{\begin{array}{ll}
\mathcal{X}_{n+1}=d\lambda_1^2\mathcal{X}_n+\frac{d(d-1)}{2}\left(-4\lambda_1^4\mathcal{X}_n^2+8\lambda_1^2\lambda_2^2\mathcal{X}_n\mathcal{Z}_n\right)
+R_x+R_z+V_x
\\
\
\\
\mathcal{Z}_{n+1}=d\lambda_2^2\mathcal{Z}_n+\frac{d(d-1)}{2}\left[\lambda_1^4\mathcal{X}_n^2-8\lambda_2^4\mathcal{Z}_n^2+\frac14\lambda_3^4(\overline{x}_n-\overline{y}_n)^2\right]
-R_z+V_z
\end{array}
\right.
\end{equation}
where

$$R_x=\mathbf{E}\left(\frac{Z_1(n)}{\sum_{i=1}^{4}Z_i(n)}-\frac{1}{4}\right)\frac{\left(\sum_{i=1}^{4}Z_i(n)-4\right)^2}{16},$$
$$R_z=\mathbf{E}\left(\frac{Z_{3}}{\sum_{i=1}^{4}Z_i(n)}-\frac{1}{4}\right)\frac{\left(\sum_{i=1}^{4}Z_i(n)-4\right)^2}{16},$$
$$
|V_x|, |V_z|\leq
C_Vx_n^2\left(\left|\frac{u_n}{x_n}-\frac{1}{4}\right|+\left|\frac{w_n}{x_n}-\frac{1}{4}\right|+x_n\right)+C_V\overline{x}_{n}^2\left(\left|\frac{\overline{w}_{n}}{\overline{x}_{n}}-\frac14\right|+\overline{x}_{n}\right)
$$
where $C_V$ is an absolute constant. 

\section{Concentration analysis}
\label{Sec:Concentration_Analysis}

In order to study the stability of the dynamical system \eqref{eq:Z_dynamics}, we
show that $R_x$, $R_z$, $V_x$, and $V_z$ are just small perturbations, in the following two lemmas. The proof of Lemma \ref{R} resembles that of Lemma $9$ in \cite{liu2018big} and is skipped for conciseness.
\begin{lemma}
	\label{R} 	Assume $|\lambda_2|\geq\varrho>0$ and $|\lambda_1|/|\lambda_2|\geq\kappa$ for
	some $\kappa>1$. For any $\varepsilon>0$, there exist $N=N(\kappa, \varepsilon)$ and $\delta=\delta(\kappa, \varrho,
	\varepsilon)>0$, such that if $n\geq N$ and $\overline{x}_n\leq x_n\leq\delta$, then
	\begin{equation*}
	|R_x|, |R_z|\leq\varepsilon x_n^2.
	\end{equation*}
\end{lemma}

The following lemma improves the result of Lemma \ref{lemma1} (c) by
establishing the strict positivity of the sum of $x_n$ and $z_n$.
\begin{lemma}
	\label{non0} Assume $\lambda_1\neq 0$. For any nonnegative $n\in
	\mathbb{Z}$, we always have
	$$
	x_n+z_n>0.
	$$
\end{lemma}
\begin{proof}
	In Lemma \ref{lemma1} we proved that $x_n+z_n\geq 0$, so it
	suffices to exclude the equality. Now let us apply reductio ad absurdum
	and assume $x_n+z_n=0$ for some $n\in \mathbb{N}$. Similar to the derivation in Lemma \ref{lemma1} (a) and (b), one can obtain that
	\begin{eqnarray*}
		\mathbf{E}(X_1(n)-X_2(n))^2=2\mathbf{E}(X_1(n))^2-2\mathbf{E}X_1(n)X_2(n)
		=x_n+z_n = 0.
	\end{eqnarray*}
	For any configuration set $A$ on the $n$th level, we always have
	$$
	\mathbf{P}(\sigma_\rho=1\mid\sigma(n)=A)=\mathbf{P}(\sigma_\rho=2\mid\sigma(n)=A).
	$$
	Denote the leftmost vertex on the $n$th level by $v_n(1)$, and it
	follows that
	$$
	\mathbf{P}(\sigma_\rho=1\mid\sigma_{v_n(1)}=1)=\mathbf{P}(\sigma_\rho=2\mid\sigma_{v_n(1)}=1).
	$$
	Define the transition matrices at distance $s$ by
	$U_s=M_{1, 1}^s$, $V_s=M_{1, 2}^s$, and $W_s=M_{1, 3}^s$,
	and then we have the following recursive system
	\begin{eqnarray*}
		\left\{\begin{array}{ll} U_s=p_0U_{s-1}+p_1V_{s-1}+2p_2W_{s-1}
			\\
			V_s=p_1U_{s-1}+p_0V_{s-1}+2p_2W_{s-1}.
		\end{array}
		\right.
	\end{eqnarray*}
	The difference of the above two equations evolves as
	$$
	U_s-V_s=\lambda_1(U_{s-1}-V_{s-1}),
	$$
	and then considering that $U_0=1$ and $V_0=W_0=0$, we have
	\begin{equation}
	\label{U-V} U_s-V_s=\lambda_1^s.
	\end{equation}
	Finally, from the reversible property of the channel, we can conclude
	that
	$$
	\lambda_1^n=U_n-V_n=\mathbf{P}(\sigma_\rho=1\mid\sigma_{v_n(1)}=1)-\mathbf{P}(\sigma_\rho=2\mid\sigma_{v_n(1)}=1)=0,
	$$
	i.e., $\lambda_1=0$, a contradiction to the assumption that
	$\lambda_1\neq0$.
\end{proof}

The following lemma ensures that $x_n$ does not drop too fast.
\begin{lemma}
	\label{ndtf}
	Suppose that there exists an integer $N>0$, such that $x_{n}\geq \overline{x}_{n}$ when $n\geq N$.	For any $\varrho>0$, if $\min\{|\lambda_1|, |\lambda_2|\}\geq\varrho$, then there exists a constant $\gamma=\gamma(\varrho, N)>0$ such that
	$$
	x_{n+1}\geq \gamma x_n.
	$$
\end{lemma}

\begin{proof}
Different to the definition of $Y_{ij}(n)=f_n(i, \sigma_j^1(n+1))$ which is the posterior probability that $\sigma_{u_j}$ takes value $i$ given the random configuration $\sigma^1_j (n + 1)$ on spins in $\mathbb{T}_{u_j} \cap L(n+1)$, we consider a configuration set $A$ on $\mathbb{T}_{u_1}\cap L(n+1)$ and define the posterior function $g_{n+1}(1, A)$ as 
	\begin{eqnarray*}
		g_{n+1}(1, A)&=&\mathbf{P}(\sigma_\rho=1\mid\sigma_1(n+1)=A)
		\\
		&=&\frac{1}{4}+p_0\left(f_n(1,
		A)-\frac{1}{4}\right)+p_1\left(f_n(2,
		A)-\frac{1}{4}\right)+p_2\sum_{i=3, 4}\left(f_n(i, A)-\frac{1}{4}\right)
		\\
		&=&\frac{1}{4}+\frac{\lambda_2+\lambda_1}{2}\left(f_n(1,
		A)-\frac{1}{4}\right)+\frac{\lambda_2-\lambda_1}{2}\left(f_n(2,
		A)-\frac{1}{4}\right).
	\end{eqnarray*}
	Setting $A=\sigma_1^1(n+1)$, by Lemma \ref{lemma:Yproperties}, we have
	\begin{eqnarray*}
		\mathbf{E}g_{n+1}(1, \sigma_1^1(n+1))
		&=&\frac{1}{4}+\frac{\lambda_2+\lambda_1}{2}\mathbf{E}\left(Y_{11}(n)-\frac{1}{4}\right)+\frac{\lambda_2-\lambda_1}{2}\mathbf{E}\left(Y_{21}(n)-\frac{1}{4}\right)
		\\
		&=&\frac{1}{4}+\lambda_1^2x_n+(\lambda_1^2-\lambda_2^2)z_n.
	\end{eqnarray*}
	
	Apparently, we have the following inequalities (see \cite{mezard2006reconstruction}), regarding the estimator $g_{n+1}(1,\sigma_1^1(n+1))$ and the
	maximum-likelihood estimator:
	\begin{eqnarray*}
		\mathbf{E}\mathbf{P}(\sigma_\rho=1\mid \sigma_1^1(n+1))&\leq&\mathbf{E}\max_{1\leq i\leq 4}\mathbf{P}(\sigma_\rho=i\mid\sigma(n+1))
		=\mathbf{E}\max_{1\leq i\leq 4}X_i(n+1)
		\\
		&\leq&\frac{1}{4}+\left(\mathbf{E}\max_i\left(X_i(n+1)-\frac{1}{4}\right)^2\right)^{1/2}
		\\
		&\leq&\frac{1}{4}+\left(\mathbf{E}\sum_{i=1}^{4}\left(X_i(n+1)-\frac{1}{4}\right)^2\right)^{1/2}
		\\
		&\leq&\frac{1}{4}+x_{n+1}^{1/2},
	\end{eqnarray*}
	where the last inequality follows from the condition that $\overline{x}_{n+1}\leq x_{n+1}$. Therefore, 
	$$\frac{1}{4}+\lambda_1^2x_n+(\lambda_1^2-\lambda_2^2)z_n\leq\frac{1}{4}+x_{n+1}^{1/2}.$$
	If $\lambda_1^2\geq \lambda_2^2$, then it is concluded from $x_n\geq -z_n\geq
	0$ in Lemma \ref{lemma1} that
	\begin{eqnarray*}
		\lambda_2^2x_n
		\leq\lambda_2^2x_n+(\lambda_1^2-\lambda_2^2)(x_n+z_n)
		=\lambda_1^2x_n+(\lambda_1^2-\lambda_2^2)z_n
		\leq x_{n+1}^{1/2}.
	\end{eqnarray*}
	If $\lambda_1^2\leq\lambda_2^2$, then
	$\lambda_1^2x_n\leq x_{n+1}^{1/2}$, since $z_n\leq0$. To sum up,
	we always have
	\begin{equation}
	\label{eqn:xiteration} 
	\min\{\lambda_1^2, \lambda_2^2\}x_n\leq x_{n+1}^{1/2}.
	\end{equation}
	
	Under the condition that $x_{n+1}\geq \overline{x}_{n+1}$, it can be concluded from the dynamical system \eqref{eq:Z_dynamics}, Lemma \ref{R}, and the following inequalities achieved in Lemma \ref{lemma1}
	\begin{equation}
	\label{unxn1} \left|\frac{u_n}{x_n}-\frac{1}{4}\right|\leq1\quad\text{and}\quad
	\left|\frac{w_n}{x_n}-\frac{1}{4}\right|\leq1,
	\end{equation}
	that there exists a $\delta=\delta(q, \varepsilon)>0$ such that
	when $x_n<\delta$ one has
	$$
	\mathcal{X}_{n+1}+\mathcal{Z}_{n+1}=x_{n+1}\geq (d\min\{\lambda_1^2, \lambda_2^2\}-\varepsilon)x_n.
	$$
	Under the condition that $\min\{|\lambda_1|,
	|\lambda_2|\}\geq\varrho$ for any $\varrho>0$, set $\varepsilon=\varrho^2$ and then we further obtain
	$$(d\min\{\lambda_1^2, \lambda_2^2\}-\varepsilon)x_n\geq
	(d-1)\varrho^2x_n\geq\varrho^2x_n.
	$$
	On the other hand, if $x_n\geq \delta$, by equation \eqref{eqn:xiteration}, one has
	$$x_{n+1}\geq(\min\{\lambda_1^2,
	\lambda_2^2\}x_n)^2\geq\varrho^4\delta x_n.$$ 
	
	Finally, by Lemma~\ref{non0}, it follows that $x_n\geq x_n+z_n>0$, and thus $\frac{x_{n+1}}{x_n}>0$ for all $n$. Therefore, taking
	$$\displaystyle\gamma=\gamma(\varrho, N)=\min_{n=0, 1, 2, \ldots, N}\left\{\varrho^2, \varrho^4\delta, \frac{x_{n+1}}{x_n}\right\}>0$$ 
	completes the proof.
\end{proof}

The following lemma provides the crucial concentration estimates of 
$u_n-\frac{x_n}{4}$ and $w_n-\frac{x_n}{4}$, when $x_n$ is small.
\begin{lemma}
	\label{unconcentration} Assume $|\lambda_2|\geq\varrho>0$ and
	$|\lambda_1|/|\lambda_2|\geq\kappa$ for
	some $\kappa>1$. For any $\varepsilon>0$, there exist $N=N(\kappa, \varepsilon)$ and $\delta=\delta(\kappa, \varrho,
	\varepsilon)>0$, such that if $n\geq N$ and $\overline{x}_n\leq x_n\leq\delta$, one has
	$$
	\left|\frac{u_n}{x_n}-\frac{1}{4}\right|<\varepsilon, \quad
	\left|\frac{w_n}{x_n}-\frac{1}{4}\right|<\varepsilon \quad \textup{and} \quad \left|\frac{\overline{w}_{n}}{\overline{x}_{n}}-\frac14\right|<\varepsilon.
	$$
	As a result, we have the estimates
	\begin{equation*}
	|V_x|, |V_z|\leq\varepsilon x_n^2.
	\end{equation*}
	
\end{lemma}
\begin{proof}
	It follows from~\ref{xnun} (d) and (e) that 
	\begin{equation*}
	\begin{split}
	&\displaystyle \mathbf{E}\left(f_n(3, \sigma^1(n))-\frac14\right)\left(f_n(4, \sigma^1(n))-\frac14\right)\\
	=&\frac14\overline{y}_{n}+\frac12\left(u_{n}-\frac14x_n\right)+\frac32\left(v_{n}-\frac14x_n\right)-\left(\overline{w}_{n}-\frac14\overline{x}_{n}\right)\\
	\end{split}
	\end{equation*}
	and
	\begin{equation*}
	\begin{split}
	\displaystyle \mathbf{E}\left(f_n(3, \sigma^1(n))-\frac14\right)\left(f_n(4, \sigma^1(n))-\frac14\right)
	=&\frac14\overline{y}_n-\left(\overline{v}_{n}-\frac14\overline{x}_n\right).
	\end{split}
	\end{equation*}
	Then by Lemma \ref{lemma1} (a) we have
	\begin{equation}
	\label{identity2}
	\left(v_{n}-\frac14x_n\right)-\left(w_{n}-\frac14x_n\right)+\left(\overline{v}_{n}-\frac14\overline{x}_n\right)-\left(\overline{w}_{n}-\frac14\overline{x}_n\right)=0.
	\end{equation}
	By the definitions of $v_{n}$, $w_{n}$, $\overline{v}_{n}$, and $\overline{w}_{n}$, and by symmetry, it follows that
	\begin{equation}
	\label{VW}
	\left(v_{n}-\frac14x_n\right)-\left(w_{n}-\frac14x_n\right)=0\quad\textup{and}\quad \left(\overline{v}_{n}-\frac14\overline{x}_n\right)-\left(\overline{w}_{n}-\frac14\overline{x}_n\right)=0.
	\end{equation}

	Plugging $a=\left(Z_1(n)-\frac{1}{4}\sum_{i=1}^{4}Z_i(n)\right)^2$, $r=\left(\left(\sum_{i=1}^{4}Z_i(n)\right)^2-16\right)$, and $s=\frac{1}{16}$ into equation \eqref{identity}, we have
	\begin{equation}
	\label{u}
	\begin{aligned}
	u_{n+1}=&\mathbf{E}\frac{\left(Z_1(n)-\frac{1}{4}\sum_{i=1}^{4}Z_i(n)\right)^2}{\left(\sum_{i=1}^{4}Z_i(n)\right)^2}
	\\
	=&\frac{1}{16}\mathbf{E}\left(Z_1(n)-\frac{1}{4}\sum_{i=1}^{4}Z_i(n)\right)^2\\
	&-\frac{1}{256}\mathbf{E}\left(Z_1(n)-\frac{1}{4}\sum_{i=1}^{4}Z_i(n)\right)^2\left(\left(\sum_{i=1}^{4}Z_i(n)\right)^2-16\right)
	\\
	&+\frac{1}{256}\mathbf{E}\frac{\left(Z_1(n)-\frac{1}{4}\sum_{i=1}^{4}Z_i(n)\right)^2}{\left(\sum_{i=1}^{4}Z_i(n)\right)^2}\left(\left(\sum_{i=1}^{4}Z_i(n)\right)^2-16\right)^2.
	\end{aligned}
	\end{equation}
	The first expectation of equation \eqref{u} will contribute to the major terms of the expansion:
	\begin{equation*}
	\begin{aligned}
	&\mathbf{E}\left(Z_1(n)-\frac{1}{4}\sum_{i=1}^{4}Z_i(n)\right)^2\\
	=&\mathbf{E}(Z_1(n)-1)^2-\frac{1}{2}\mathbf{E}(Z_1(n)-1)\left(\sum_{i=1}^{4}Z_i(n)-4\right)+\frac{1}{16}\mathbf{E}\left(\sum_{i=1}^{4}Z_i(n)-4\right)^2
	\\
	=&4d\lambda_1^2x_n+4d(\lambda_1^2-\lambda_2^2)z_n+16d\lambda_1^2\lambda_2\left(u_n-\frac{x_n}{4}\right)+O(x_n^2),
	\end{aligned}
	\end{equation*}
	where Lemma \ref{lemma:Z} is used in the last equity and the following derivations. 
	Similarly, we can bound both the second and third terms of equation \eqref{u} by $O(x_n^2)$:
	\begin{eqnarray*}
		\mathbf{E}\left(Z_1(n)-\frac{1}{4}\sum_{i=1}^{4}Z_i(n)\right)^2\left(\left(\sum_{i=1}^{4}Z_i(n)\right)^2-16\right)&=&O(x_n^2),
	\end{eqnarray*}
	and
	\begin{eqnarray*}
		\mathbf{E}\left(\left(\sum_{i=1}^{4}Z_i(n)\right)^2-16\right)^2&=&O(x_n^2).
	\end{eqnarray*}
	Considering that $\mathcal{X}_n=x_n+\overline{z}_n$ and $\mathcal{Z}_n=-\overline{z}_n$, the dynamical system \eqref{eq:Z_dynamics} yields that 
	$$x_{n+1}=d\lambda_1^2x_n+d(\lambda_1^2-\lambda_2^2)z_n+O(x_n^2).$$
	Equation \eqref{u} gives
	\begin{equation}
	\label{unxn}
	u_{n+1}=\frac{x_{n+1}}{4}+d\lambda_1^2\lambda_2\left(u_n-\frac{x_n}{4}\right)+O(x_n^2),
	\end{equation}
	and then
	\begin{equation}
	\label{w}
	\frac{u_{n+1}}{x_{n+1}}-\frac{1}{4}=d\lambda_1^2\lambda_2\frac{x_n}{x_{n+1}}\left(\frac{u_n}{x_n}-\frac{1}{4}\right)+O\left(\frac{x_n^2}{x_{n+1}}\right).
	\end{equation}
	
	Next display the discussion in the $\mathcal{X}O\mathcal{Z}$ plane. First consider the
	case that $|\lambda_1|/|\lambda_2|\geq\kappa$ for $\kappa>1$. In a small neighborhood of $(0, 0)$,
	since $d\lambda_2^2<\kappa^2d|\lambda_2^2|\leq d\lambda_1^2<1$
	and $\mathcal{X}_n>0$, the discrete trajectory approaches the origin point in a way that is 
	``tangential" to the $\mathcal{X}$-axis, when $x_n$ is small enough (see \cite{bernussou1977point}). Furthermore, the conclusion of Lemma \ref{non0} excludes the possibility that the trajectory moves along the $\mathcal{Z}$-axis. Then for some $M>1$, there exist
	constants $N_1=N_1(\kappa, M)$ and $\delta_1=\delta_1(\kappa, M)$, such that if $n\geq N_1$ and $x_n\leq\delta_1$, we have
	$$\mathcal{X}_n\geq M\mathcal{Z}_n \quad \text{and} \quad \frac{1}{M(M+1)}d\lambda_1^2x_n+O(x_n^2)>0,$$ where the remainder
	term $O(x_n^2)$ comes from the expansion of $x_{n+1}$.
	Consequently, it follows 
	$$x_n+z_n=\mathcal{X}_n\geq \frac{M}{M+1}(\mathcal{X}_n+\mathcal{Z}_n)=\frac{M}{M+1}x_n,$$ 
	and by the fact that $z_n\leq 0$ then
	\begin{equation}
	\label{ratio}
	\begin{aligned}
	\frac{x_n}{x_{n+1}}=&\frac{x_n}{d\lambda_1^2x_n+d(\lambda_1^2-\lambda_2^2)z_n+O(x_n^2)}
	\leq\frac{x_n}{\frac{M}{M+1}d\lambda_1^2x_n+O(x_n^2)}\\
	\leq&\frac{x_n}{\left(1-\frac1M\right)d\lambda_1^2x_n}
	=\frac{M}{M-1}\frac{1}{d\lambda_1^2}.
	\end{aligned}
	\end{equation}
	
	For fixed $k$, by the fact that $\frac14\lambda_3^4(\overline{x}_n-\overline{y}_n)^2$ can be bounded by $O(x_n^2)$ for the reason that $|\overline{x}_n|>|\overline{y}_n|$ implied in Lemma \ref{lemma1} (b) and (c),
	it is known from the dynamical system \eqref{eq:Z_dynamics} that
	\begin{eqnarray*}
		|x_{n+1}-(d\lambda_1^2\mathcal{X}_n+d\lambda_2^2\mathcal{Z}_n)|\leq Cx_n^2.
	\end{eqnarray*}
	Furthermore, one has
	$$
	x_{n+1}\leq(d\lambda_1^2\mathcal{X}_n+d\lambda_2^2\mathcal{Z}_n)+Cx_n^2\leq (d\lambda_1^2+Cx_n)x_n,
	$$
	and then there exists $\delta_2=\delta_2(\kappa, M,
	k)<\delta_1$, such that if $x_n<\delta_2$ then for any $1\leq\ell\leq k$ one has
	$x_{n+\ell}<2\delta_2$. Therefore, for any
	positive integer $k$,  equation \eqref{w} yields
	\begin{eqnarray*}
		\frac{u_{n+k}}{x_{n+k}}-\frac{1}{4}&=&d\lambda_1^2\lambda_2\frac{x_{n+k-1}}{x_{n+k}}\left(\frac{u_{n+k-1}}{x_{n+k-1}}-\frac{1}{4}\right)+O\left(x_{n+k-1}\frac{x_{n+k-1}}{x_{n+k}}\right)
		\\
		&=&(d\lambda_1^2\lambda_2)^k\left(\prod_{\ell=1}^k\frac{x_{n+\ell-1}}{x_{n+\ell}}\right)\left(\frac{u_n}{x_n}-\frac{1}{4}\right)+R,
	\end{eqnarray*}
	where, by equation \eqref{w} and with $C$ denoting the $O$ constant therein,
	$$
	|R|\leq
	2C\delta_2\left(\sum_{i=1}^k\left(\frac{M}{M-1}\frac{1}{d\lambda_1^2}\right)^i(d\lambda_1^2|\lambda_2|)^{i-1}\right)
	\leq\delta_2\frac{1-\left(\frac{M}{M-1}|\lambda_2|\right)^k}{1-\left(\frac{M}{M-1}|\lambda_2|\right)}\frac{M}{M-1}\frac{1}{d\lambda_1^2},
	$$
	and by equation \eqref{ratio}
	\begin{eqnarray*}
		(d\lambda_1^2\lambda_2)^k\left(\prod_{\ell=1}^k\frac{x_{n+\ell-1}}{x_{n+\ell}}\right)
		\leq(d\lambda_1^2|\lambda_2|)^k\left(\frac{M}{M-1}\frac{1}{d\lambda_1^2}\right)^k=
		\left(\frac{M}{M-1}|\lambda_2|\right)^k.
	\end{eqnarray*}

	Firstly, from Lemma \ref{lemma1} (a) one has
	$0\leq\frac{u_n}{x_n}\leq1$, which implies that
	$\left|\frac{u_n}{x_n}-\frac{1}{4}\right|< 1$. Secondly, by the fact that $|\lambda_2|\leq |\lambda_1|\leq d^{-1/2}\leq
	1/\sqrt{2}$, it is possible to achieve $\frac{M}{M-1}|\lambda_2|<1$
	by choosing $M=4$. Therefore, we can conclude that it is
	feasible to take $k=k(\varepsilon)$ sufficiently large and
	$\delta_3=\delta_3(\kappa, k, \varepsilon)=\delta_3(\kappa,
	\varepsilon)<\delta_2$ sufficiently small to guarantee that
	$$\left|\frac{u_{n+k}}{x_{n+k}}-\frac{1}{4}\right|<\varepsilon.$$
	Finally, under the condition that $|\lambda_2|\geq \varrho>0$, by Lemma \ref{ndtf}, we know that there exists
	$\gamma=\gamma(\varrho)$ such that
	$x_{n-k}\leq \gamma^{-k}x_n$. Thus, we can choose $N=N(\kappa,
	\varepsilon, k)=N(\kappa, \varepsilon)>N_1+k$ and
	$\delta=\gamma^k\delta_3$, such that if $x_n\leq\delta$ and $n\geq N$ then
	\begin{equation}
	\label{wnxn} \left|\frac{u_n}{x_n}-\frac{1}{4}\right|< \varepsilon.
	\end{equation}
 The second part of the lemma can be shown similarly as
	above.
\end{proof}

\section {Proof of the Main Theorem}
\label{Sec:Proof_of_Main_Theorem}

First, consider
$\varrho\leq |\lambda_2|\leq |\lambda_1|$ for any fixed $\varrho>0$.
To investigate the non-tightness, it would be convenient to assume that $1>d\lambda_1^2\geq d\lambda_2^2\geq\frac12$, say,
$|\lambda_1|\geq\frac{1}{\sqrt{2d}}$. We take $\varrho=\frac{1}{\sqrt{2d}}$ in the following context. Consider $|\lambda_2|>\varrho$ fixed and
just $\lambda_1$ varying, and without loss of generality,
assume $d\lambda_1^2>\frac{1+d\lambda_2^2}{2}$. Consequently choose
$\kappa=\kappa(d,
\lambda_2)=\left(\frac{1+d\lambda_2^2}{2d\lambda_2^2}\right)^{1/2}>1$
and thus $|\lambda_1|/|\lambda_2|\geq \kappa$.

By the definition of non-reconstruction in equation \eqref{equivalent}, it
suffices to show that when $d\lambda_1^2$ is close enough to
$1$, $\mathcal{X}_n$ does not converge to $0$ for the reason that it implies that $x_n$ does not
converge to $0$ considering $0\leq \mathcal{X}_n=x_n+z_n\leq x_n$. We apply reductio ad absurdum, by assuming that
\begin{equation}
\label{assumption}
\lim_{n\to\infty}x_n=\lim_{n\to\infty}\overline{x}_{n}=0.
\end{equation} 
Therefore, there exists $\mathcal{N}_1=\mathcal{N}_1(d)$, such that whenever $n>\mathcal{N}_1$, we have $x_n\leq\delta$. Next, recalling that $\mathcal{X}_n=x_n+\overline{z}_n$, we further define $\overline{\mathcal{X}}_n=\overline{x}_n+\overline{z}_n$. 
Then by the symmetry of the model, we can obtain the dynamical form for $\overline{\mathcal{X}}_n$ analogously as the dynamical form for $\mathcal{X}_n$ in equation \eqref{eq:Z_dynamics} :
$$
\overline{\mathcal{X}}_{n+1}=d\lambda_3^2\overline{\mathcal{X}}_n+\frac{d(d-1)}{2}\left(-4\lambda_3^4\overline{\mathcal{X}}_n^2+8\lambda_3^2\lambda_2^2\overline{\mathcal{X}}_n\mathcal{Z}_n\right)
+R_{\overline{x}}+R_z+V_{\overline{x}}
$$
where $R_{\overline{x}}$ and $V_{\overline{x}}$ are counterparts of $R_x$ and $V_x$ simply by replacing $x$ by $\overline{x}$.

Then we display the discussion in the $\mathcal{X}O\overline{\mathcal{X}}$ plane. Since $|\lambda_1|>|\lambda_3|$ and $\mathcal{X}_n, \overline{\mathcal{X}}_n\to 0$ as $n\to \infty$ from equation \eqref{assumption}, in a small neighborhood of $(0, 0)$, the discrete trajectory approaches the origin point in a way that is 
``tangential" to the $\mathcal{X}$-axis. Furthermore, the conclusion of Lemma \ref{non0} excludes the possibility that the trajectory moves along the $\overline{\mathcal{X}}$-axis. Therefore, it implies that there exists $\mathcal{N}=\mathcal{N}(d)>\mathcal{N}_1$, such that whenever $n>\mathcal{N}$,
\begin{equation}
\label{newinequality}
\overline{\mathcal{X}}_n\leq \mathcal{X}_n,\quad \text{that is,}\quad \overline{x}_n\leq x_n.
\end{equation}

From the proof of Lemma \ref{unconcentration}, we know that
in the $\mathcal{X}O\mathcal{Z}$ plane there exist
$N=N(\kappa, \varrho)>\mathcal{N}$ and $\delta=\delta(d, \kappa,
\varrho)>0$, such that if $n\geq N$ and $x_n\leq\delta$, then in the
small neighborhood of $(0, 0)$, we have
\begin{equation}
\label{6.1}
\mathcal{X}_n\geq 4\mathcal{Z}_n \quad\textup{that is,}\quad
\mathcal{\mathcal{X}}_n\geq\frac45x_n
\end{equation} 
By equation \eqref{newinequality}, applying Lemma~\ref{R}, and taking $\varepsilon=\frac{4}{25}\frac{d(d-1)}4\lambda_1^4$, one can obtain
$$|R_z|\leq\frac{4}{25}\frac{d(d-1)}4\lambda_1^4x_n^2\leq
\frac{1}{4}\frac{d(d-1)}{4}\lambda_1^4\mathcal{X}_n^2.$$
Next by the result of Lemma \ref{unconcentration} that $\left|\frac{u_n}{x_n}-\frac{1}{4}\right|<\varepsilon'$ and $\left|\frac{w_n}{x_n}-\frac{1}{4}\right|<\varepsilon'$  for any $\varepsilon'>0$, now we take 
$\varepsilon'=\frac{1}{12C_V}\frac{d(d-1)}4\lambda_1^4.$
Therefore, by equation \eqref{eq:Z_dynamics} and the condition that $\lambda_1\geq \lambda_2$, we have
\begin{eqnarray}
\label{major}
\mathcal{Z}_{n+1} &=&
d\lambda_2^2\mathcal{Z}_n+\frac{d(d-1)}{2}\left[\lambda_1^4\mathcal{X}_n^2-8\lambda_2^4\mathcal{Z}_n^2+\frac14\lambda_3^4(\overline{x}_n-\overline{y}_n)^2\right]
-R_z+V_z\nonumber
\\
&\geq&	d\lambda_2^2\mathcal{Z}_n+\frac{d(d-1)}{2}\left[\lambda_1^4\mathcal{X}_n^2-8\lambda_2^4\mathcal{Z}_n^2\right]
-R_z+V_z\nonumber
\\
&\geq&	d\lambda_2^2\mathcal{Z}_n+\frac{d(d-1)}{2}\left[\frac12\lambda_1^4\mathcal{X}_n^2+\frac12\lambda_1^416\mathcal{Z}_n^2-8\lambda_2^4\mathcal{Z}_n^2\right]
-R_z+V_z
\\
&\geq&d\lambda_2^2\mathcal{Z}_n+\frac{d(d-1)}4\lambda_1^4\mathcal{X}_n^2-|R_z|
-C_Vx_n^2\left(\left|\frac{u_n}{x_n}-\frac{1}{4}\right|+\left|\frac{w_n}{x_n}-\frac{1}{4}\right|+x_n\right),\nonumber
\\
&\geq&d\lambda_2^2\mathcal{Z}_n+\frac12\frac{d(d-1)}{4}\lambda_1^4\mathcal{X}_n^2,\nonumber
\\
&\geq&\mathcal{Z}_n\left[d\lambda_2^2+\frac{d(d-1)}{2}\lambda_1^4\mathcal{X}_n\right].\nonumber
\end{eqnarray}

Note that the initial point $x_0=1-\frac14=\frac34>0$ and Lemma~\ref{ndtf} implies that there exists $\gamma=\gamma(\varrho, \mathcal{N})=\gamma(d)$ such that
$ x_n\geq x_0\gamma^n$. Define $\varepsilon=\varepsilon(d)=\left(\frac{x_0\gamma^{N}}{10}\right)^2> 0$. Because $\varepsilon$ is independent of
$\lambda_1$, considering that $d\lambda_2^2$ sufficiently close to $1$, we can choose $|\lambda_1|<d^{-1/2}$ such that
\begin{equation}
\label{inequality}
d\lambda_2^2+\frac{d(d-1)}{2}\lambda_1^4\varepsilon>1.
\end{equation}
Noting that $\frac{d(d-1)}{2}\lambda_1^4\geq\left(\frac{d\lambda_1^2}2\right)^2\geq\frac1{16}$, equation \eqref{major} implies that 
$$
\mathcal{Z}_{N+1}\geq
\frac12\frac{d(d-1)}{4}\lambda_1^4\mathcal{X}_N^2\geq \frac14\frac1{16}\frac{16}{25}x_n^2\geq\left(\frac{x_0\gamma^N}{10}\right)^2=
\varepsilon.
$$ 
Suppose $\mathcal{Z}_n\geq \varepsilon$ for some
$n> N$, and it follows from equations \eqref{major}
and \eqref{inequality} that
\begin{eqnarray*}
	x_{n+1}\geq \mathcal{Z}_{n+1} \geq
	\mathcal{Z}_n\left[d\lambda_2^2+\frac{d(d-1)}{2}\lambda_1^4\varepsilon\right]
	> \mathcal{Z}_n \geq\varepsilon.
\end{eqnarray*}
Therefore, by induction we have $x_n\geq \mathcal{Z}_n \geq\varepsilon$ for all $n>N$,
which contradicts to the assumption imposed in equation \eqref{assumption}. Thus, the proof is completed.

\bibliographystyle{siamplain}
\bibliography{references}
\end{document}